\documentclass[letterpaper, 10 pt, conference]{ieeeconf}
\usepackage{times}
\usepackage[pdftex]{graphicx}
\usepackage{}
\usepackage{amsmath,amssymb,amsopn,amstext,amsfonts}
\usepackage{pifont}
\newcommand{\cmark}{\ding{51}}%
\newcommand{\xmark}{\ding{55}}%
\usepackage{cancel}
\usepackage[space]{cite}
\usepackage{pdfsync}
\usepackage{balance}
\usepackage{color}
\usepackage{mathtools}
\usepackage{algpseudocode}

\usepackage[ruled,vlined,linesnumbered]{algorithm2e}

\algnewcommand\algorithmicforeach{\textbf{for each}}
\algdef{S}[FOR]{ForEach}[1]{\algorithmicforeach\ #1\ \algorithmicdo}

\usepackage{bm}

\newtheorem{lemma}{Lemma}
\usepackage{diagbox}
\usepackage{pifont}
\usepackage{amsmath}
\usepackage{multirow}
\usepackage{url}
\usepackage{verbatim}
\usepackage{footmisc}
\usepackage[linkcolor=black,citecolor=black,urlcolor=black,colorlinks=true]{hyperref}
\usepackage{color}
\usepackage{tabularx}
\usepackage{float}
\graphicspath{{figures/}}
\DeclareGraphicsExtensions{.png,.jpg,.eps}
\IEEEoverridecommandlockouts
\title{\LARGE \bf
    ikd-Tree: An Incremental K-D Tree for Robotic Applications
}

\author{Yixi Cai, Wei Xu and Fu Zhang 
	\thanks{Y. Cai, W. Xu and F. Zhang are with the Department of Mechanical Engineering, Hong Kong University, Hong Kong SAR., China. {\tt\small $\{$yixicai$\}$@connect.hku.hk, $\{$xuweii, fuzhang$\}$@hku.hk}
	}
}%

\begin{document}
	\maketitle
	\newcommand{\note}[1]{\textcolor{red}{\emph{\bf#1}}}
	\newcommand\footnoteref[1]{\protected@xdef\@thefnmark{\ref{#1}}\@footnotemark}
	\newlength{\bibitemsep}\setlength{\bibitemsep}{.0238\baselineskip}
	\newlength{\bibparskip}\setlength{\bibparskip}{0pt}
	\let\oldthebibliography\thebibliography
	\renewcommand\thebibliography[1]{%
		\oldthebibliography{#1}%
		\setlength{\parskip}{\bibitemsep}%
		\setlength{\itemsep}{\bibparskip}%
	}
	\begin{abstract}
    This paper proposes an efficient data structure, ikd-Tree, for dynamic space partition. The ikd-Tree incrementally updates a k-d tree with new coming points only, leading to much lower computation time than existing static k-d trees. Besides point-wise operations, the ikd-Tree supports several features such as box-wise operations and down-sampling that are practically useful in robotic applications. In parallel to the incremental operations (i.e., insert, re-insert, and delete), ikd-Tree actively monitors the tree structure and partially re-balances the tree, which enables efficient nearest point search in later stages. The ikd-Tree is carefully engineered and supports multi-thread parallel computing to maximize the overall efficiency. We validate the ikd-Tree in both theory and practical experiments. On theory level, a complete time complexity analysis is presented to prove the high efficiency. On experiment level, the ikd-Tree is tested on both randomized datasets and real-world LiDAR point data in LiDAR-inertial odometry and mapping application. In all tests, ikd-Tree consumes only 4\% of the running time in a static k-d tree.

	\end{abstract}

	\section{Introduction}\label{sect_intro}
	The K-Dimensional Tree (K-D Tree) is an efficient data structure that organizes multi-dimensional point data \cite{bentley1975kdtree} which enables fast search of nearest neighbors, an essential operation that is widely required in various robotic applications\cite{friedman1975nearestsearch}. For example, in LiDAR odometry and mapping, k-d tree-based nearest points search is crucial to match a point in a new LiDAR scan to its correspondences in the map (or the previous scan) \cite{nuchter2007cachedKDtreeICP,segal2009generalizedicp, zhang2014loam, shan2018lego, lin2020loamlivox, xu2020fastlio}. Nearest points search is also important in motion planning for fast obstacle collision check on point-cloud, such as in \cite{ichnowski2015fast, gao2016online, lopez2017aggressive, florence2018nanomap, gao2019flying, ji2020mapless}. 
	
	Common-used k-d tree structure in robotic applications\cite{rusu20113d} is ``static", where the tree is built from scratch using all points. This contradicts with the fact that the data is usually acquired sequentially in actual robotic applications. In this case, incorporating a frame of new data to existing ones by re-building the entire tree from scratch is typically very inefficient and time-consuming. As a result, k-d trees are usually updated at a low frequency \cite{zhang2014loam, shan2018lego, lin2020loamlivox} or simply re-built only on the new points \cite{lopez2017aggressive, florence2018nanomap}. 
	
	
	To fit the sequential data acquisition nature, a more natural k-d tree design would be updating (i.e., insert and delete) the existing tree locally with the newly acquired data. The local update would effectively eliminate redundant operations in re-building the entire tree, and save much computation. Such a dynamic k-d tree is particularly promising when the new data is much smaller than existing ones in the tree. 
	
	
	However, a dynamic k-d tree brings suitable for robotic applications several challenges: 1) It should support not merely efficient points operations such as insertion and delete but also space operations such as point-cloud down-sampling; 2) It easily grows unbalanced after massive points or space operations which deteriorates efficiency of points  queries. Hence re-building is required to re-balance the tree. 3) The re-building should be sufficiently efficient to enable real-time robotic applications.
	
    \begin{figure}[t]
        \setlength\abovecaptionskip{-0.7\baselineskip}
        \centering
        \includegraphics[width=0.5\textwidth]{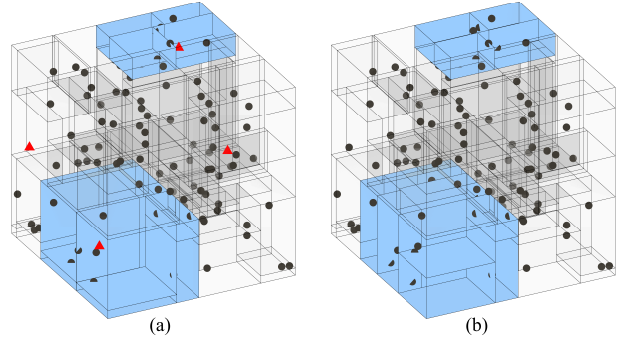}
        \caption{Illustration of incremental k-d tree update and re-balancing. (a): an existing k-d tree (black dots) and new points (red triangles) to insert, blue cubes denote the space (i.e., branches) need to be re-balanced. (b): the k-d tree after points insertion and tree re-balancing, blue cubes denote the space after re-balancing while rest majority tree does not change. 
        }
        \vspace{-0.5cm}
        \label{fig:kd_tree_space}
    \end{figure}


	In this paper, we propose a dynamic k-d tree structure called ikd-Tree, which builds and incrementally updates a k-d tree with new points only while simultaneously down-sample them into the desired resolution. It supports incremental operations including insertion, re-insertion, and delete of a single point (i.e., point-wise) or a box of points (i.e., box-wise). The tree is automatically re-balanced by partial re-building. To preserve efficient real-time tree update, ikd-Tree separates the tree re-building in two parallel threads when necessary. This paper also provides a complete time complexity analysis for all tree updates, including both incremental operations and re-balancing. The time complexity of ikd-Tree is reduced substantially as verified on both random data and real-world point-cloud in LiDAR-inertial mapping applications. The ikd-Tree is open sourced at Github\footnote{Git: \url{https://github.com/hku-mars/ikd-Tree.git}\label{footnote:git}}. Fig. \ref{fig:kd_tree_space} illustrates the incremental updates and re-balancing on the tree from a 3-D view.
	
	The remaining paper is organized as follows: Section \ref{sec:relatedwork} introduces related work. The design of ikd-Tree is described in Section \ref{sec:3}. Theoretical analysis of time and space complexity are presented in Section \ref{sec:theory}. Experiments are shown in Section \ref{sec:experiments}, followed by conclusions in Section \ref{sec:conclusion}.
	
	\section{Related work}\label{sec:relatedwork}

    A k-d tree can be viewed as a binary search tree and inherits the same  incremental operations (i.e., insert, re-insert, and delete), such as those in AVL trees \cite{bayer1972AVL}, treaps \cite{aragon1989treap} and splay trees\cite{sleator1985splay}. In these techniques, re-balancing of a binary search tree after many points operations can be easily achieved by tree node rotations. However, this straightforward tree rotation only works for one-dimensional data. For k-d trees with higher data dimension, it requires much more complicated tree reorganization.
    
    Strategies specifically designed for fast re-balancing k-d trees fall into two categories: hardware-based acceleration and specially designed structure enabling dynamic re-balancing. Hardware-based algorithms exploits the computing hardware to (re-) balance a kd-tree by building a new one. It has been thoroughly investigated to solve the ray tracing problem in dynamic scenes. In 3D graphic applications, algorithms on single-core CPU\cite{huntsinglecore, popov2006sah} and multi-core CPU \cite{shevtsovmulticore} are firstly proposed to speed up the k-d tree construction. Zhou \textit{et al.} proposed a real-time construction algorithm on GPU\cite{zhougpu}. These algorithms rely heavily on the computing resource which is usually limited on onboard computers. 
    
    For the second category, Bentley \textit{et al.} proposed a general binary transformation method for converting a static k-d tree to a dynamic one \cite{bentley1980decomposable}. The dynamic k-d tree supports only insertion but not delete, which leads to a growing tree size hence increased time for nearest search. Galperin \textit{et al.} \cite{scapegoat} proposes a scapegoat k-d tree that can dynamically re-balance the tree by re-building unbalanced sub-trees, which is much more efficient than a complete re-building of the entire tree. Bkd tree \cite{procopiuc2003bkd} is a dynamic data structure extended from a K-D-B tree\cite{robinson1981kdb} which focus on external memory adaptions. A set of static k-d trees are built in the Bkd tree where the trees are re-balanced by re-building partial set of the trees at regular intervals. The well-known point cloud library (PCL)\cite{rusu20113d} uses the fast library for approximate nearest neighbors (FLANN) search\cite{muja2009flannfast}. Point insertion and delete are supported in FLANN but the trees are re-balanced via inefficient complete tree re-building after a predetermined amount of points insertion or delete \cite{muja2009flannmanual}. 

    Our ikd-Tree is an efficient and complete data structure enabling incremental operations (i.e., insert, re-insert, and delete) and dynamic re-balancing of k-d trees. Compared to the dynamic k-d tree in \cite{bentley1980decomposable}, our implementation supports points delete. Besides the point-wise operations presented in \cite{bayer1972AVL, aragon1989treap, sleator1985splay} and \cite{scapegoat}, our ikd-Tree further supports the incremental operations of a box of points (i.e., box-wise operations) and simultaneous points down-sampling. The dynamic tree re-balancing strategy of ikd-Tree follows the concept of scapegoat trees in \cite{scapegoat}, which only re-builds those unbalanced sub-trees. The ikd-Tree is particularly suitable for robotic applications, such as real-time LiDAR mapping and motion planning, where data are sampled sequentially and fast incremental update is necessary.
    
    

	\section{ikd-Tree Design and Implementation}\label{sec:3}
    	In this section, we describe how to design, build, and update an incremental k-d tree in ikd-Tree to allow incremental operations (e.g., insertion, re-insertion, and delete) and dynamic re-balancing. 
    	
        \subsection{Data Structure}

        The attributes of a tree node in ikd-Tree is presented in \textbf{Data Structure \ref{struct}}. Line 2-4 are the common attributes for a standard k-d tree. The attributes $leftson$ and $rightson$ are pointers to its left and right son node, respectively. The point information (e.g., point coordinate, intensity) are stored in $point$. Since a point corresponds a single node on a k-d tree, we will use points and nodes interchangeably. The division axis is recorded in $axis$. Line 5-7 are the new attributes designed for incremental updates detailed in Section \ref{subsec:increupdates}. 
    	\begin{algorithm}
        \label{struct}    	
        \SetAlgoLined
        \NoCaptionOfAlgo
        \SetKwProg{kdtreenode}{Struct}{:}{end}
        \kdtreenode{\text{TreeNode}}{
        \tcp{Common Attributes in Standard K-D trees}
        PointType $point$;\\
        TreeNode * $leftson$, $rightson$;\\
        int $axis$;\\
        \tcp{New Attributes in ikd-Tree}
        int $treesize$, $invalidnum$;\\
        bool $deleted$, $treedeleted$, $pushdown$;\\        
        float $range[k][2]$;
        }
        \caption{\textbf{Data Structure 1}: Tree node structure}
        \end{algorithm}        
        \vspace{-0.3cm}
    	\subsection{Building An Incremental K-D Tree}
    	Building an incremental k-d tree is similar to building a static k-d tree except maintaining extra information for incremental updates. The entire algorithm is shown in {\bf Algorithm \ref{alg:build}}: given a point array $V$, the points are firstly sorted by the division axis with maximal covariance (Line 4-5). Then the median point is saved to $point$ of a new tree node $T$ (Line 6-7). Points below and above the median are passed to the left and right son nodes of $T$, respectively, for recursive building (Line 9-10).  The \texttt{LazyLabelInit} and \texttt{Pullup} in Line 11-12 update all attributes necessary for incremental updates (see \textbf{Data Structure \ref{struct}}, Line 5-7) detailed in Section \ref{subsec:increupdates}. 
    	\setcounter{algocf}{0}     	
        \begin{algorithm}
            \SetAlgoLined
            \DontPrintSemicolon
            \KwIn{$ V, N $\Comment Point Array and Point Number}    
            \KwOut{$ RootNode $ \Comment K-D Tree Node}
            \SetKwFunction{update}{Pullup}            
            \SetKwFunction{build}{Build}
            $RootNode$ = \build{$V,0,N-1$};\\
            \SetKwProg{Fn}{Function}{}{}
            
            \Fn{\build{$V,l,r$}}{
                $ mid  \leftarrow \lfloor (l+r)/2 \rfloor $;\\
                $ Axis \leftarrow \text{Axis with Maximal Covariance}$;\\
                $ V \leftarrow sort(V,axis)$;\\
                Node $T$;\\
                $T.point \leftarrow V[mid]$;\\ 
                $T.axis \leftarrow Axis$;\\
                $T.leftson \leftarrow$ \build{$V,l,mid-1$};\\
                $T.rightson \leftarrow$ \build{$V,mid+1,r$};\\
                \texttt{LazyLabelInit}($T$); \\
                \update{$T$};\\
                \textbf{return} $ T $;
            }
            \textbf{End Function}
            \caption{Build a balanced k-d tree}
            \label{alg:build}              
        \end{algorithm}    	
        \addtolength{\textfloatsep}{-0.2cm}
            
    	\subsection{Incremental Updates}\label{subsec:increupdates}
    	
    	The incremental updates refer to incremental operations followed by a dynamic re-balancing detailed in Section \ref{sec:balance}. The incremental operations include insertion, delete and re-insertion of points to/from the k-d tree. Specifically, the insertion operation appends a new point (i.e., a node) to the k-d tree. In the delete operation, we use a lazy delete strategy. That is, the points are not removed from the tree immediately but only labeled as ``deleted" by setting the attribute $deleted$ to true (see \textbf{Data Structure \ref{struct}}, Line 6). If all nodes on the sub-tree rooted at $T$ have been deleted, the attribute $treedeleted$ of $T$ is set to true. Therefore the attributes $deleted$ and $treedeleted$ are called lazy labels.  If points labeled as ``deleted" but not removed are later inserted to the tree, it is referred to as ``re-insertion" and is efficiently achieved by simply setting the $deleted$ attribute back to false. Otherwise, points labeled as ``deleted" will be removed from the tree during re-building process(see Section \ref{sec:balance}). 
    	
    	Our incremental updates support two types: point-wise updates and box-wise updates. The point-wise updates insert, delete, or re-insert a single point on the tree while the box-wise updates insert, delete or re-insert all points in a given box aligned with the data coordinate axis. Box-wise updates may require to delete or re-insert an entire sub-tree rooted at $T$. In this case, recursively updating the lazy labels $deleted$ and $treedeleted$ for all offspring nodes of $T$ are still inefficient. To address this issue, we use a further lazy strategy to update the lazy labels of the offspring nodes. The lazy label for lazy labels $deleted$ and $treedeleted$ is $pushdown$ (see \textbf{Data Structure \ref{struct}}, Line 6). The three labels $deleted$, $treedeleted$, and $pushdown$ are all initialized as false in \texttt{LazyLabelInit} (see {\bf Algorithm \ref{alg:build}}, Line 11).
    	
        
    	\subsubsection{Pushdown and Pullup}
    	Two supporting functions, \texttt{Pushdown} and \texttt{Pullup}, are designed to update attributes on a tree node $T$. The \texttt{Pushdown} function copies the labels $deleted$, $treedeleted$, and $pushdown$ of $T$ to its sons (but not further offsprings) when the attribute $pushdown$ is true. The \texttt{Pullup} function summarizes the information of the sub-tree rooted at ${T}$ to the following attributes of node $T$: $treesize$ (see \textbf{Data Structure \ref{struct}}, Line 5) saving the number of all nodes on the sub-tree, $invalidnum$ saving the number of nodes labelled as ``deleted" on the sub-tree, and $range$ (see \textbf{Data Structure \ref{struct}}, Line 7) summarising the range of all points on the sub-tree along coordinate $axis$, where $k$ is the points dimension.

    	\subsubsection{Point-wise Updates} The point-wise updates on the incremental k-d tree are implemented in a recursive way which is similar to the scapegoat k-d tree\cite{scapegoat}. For point-wise insertion, the algorithm searches down from the root node recursively and compare the coordinate on division axis of the new point with the points stored on the tree nodes until a leaf node is found to append a new tree node. For delete or re-insertion of a point $P$, the algorithm finds the tree node storing the point $P$ and modifies the attribute $deleted$. Further details can be found in our Github repository\footref{footnote:git}.
        \begin{algorithm}[t]
        \SetAlgoLined
            \SetKwFunction{pushdown}{Pushdown}
            \SetKwFunction{check}{CheckCriterion}
            \SetKwFunction{rebuild}{Rebuild}
            \SetKwFunction{update}{Pullup}
            \SetKwProg{Fn}{Function}{}{}
        \KwIn{$ C_O$ \Comment Operation box\newline
        $T$ \Comment K-D Tree Node\newline
        $SW$ \Comment Switch of Parallelly Re-building}
            \SetKwFunction{boxwiseop}{BoxwiseUpdate}
            \SetKwFunction{lazy}{UpdateLazyLabel}
            \SetKwFunction{pushdown}{Pushdown}
            \SetKwFunction{check}{ViolateCriterion}
            \SetKwFunction{rebuild}{Rebuild}
            \SetKwFunction{parrebuild}{ParallelRebuild}
            \SetKwFunction{update}{Pullup}
            \SetKwFunction{newthread}{ThreadSpawn}
            \SetKwProg{Fn}{Function}{}{}
            \Fn{\boxwiseop{$T,C_O,SW$}}{
                \pushdown{$T$};\\
                $C_T \leftarrow T.range$;\\
                \lIf{$C_T \cap C_O = \varnothing$ }{\textbf{return}}

                \eIf{$C_T \subseteqq C_O$}{
                    \lazy{};\\
                    $T.pushdown$ = true;\\
                    \textbf{return};
                }{
                $P\leftarrow T.point$;\\
                \lIf{$P\subset C_O$}{Modify $T.deleted$}
                \boxwiseop{$T.leftson,C_O,SW$};\\
                \boxwiseop{$T.rightson,C_O,SW$};\\
                }
                \update{$T$};\\
                \If{\check{$T$}}{
                    \eIf{$T.treesize< N_{\max}$ \textbf{or} Not $SW$}{
                        \rebuild{$T$}}{
                        \newthread{\parrebuild, $T$}
                    }
                }
            }
            \textbf{End Function}\\
        \caption{Box-wise Updates}
        \label{alg:boxwise}         
        \end{algorithm}     	
        \addtolength{\textfloatsep}{-0.2cm}        
        \subsubsection{Box-wise Updates}
        The box-wise insertion is implemented by inserting the new points one by one into the incremental k-d tree. Other box-wise updates (box-wise delete and re-insertion) are implemented utilizing the range information in attribute $range$, which forms a box $C_T$, and the lazy labels on the tree nodes. The pseudo code is shown in \textbf{Algorithm \ref{alg:boxwise}}. Given the box of points $C_O$ to be updated on (sub-) tree rooted at $T$, the algorithm first passes down its lazy labels to its sons for further passing-down if visited (Line 2). Then, it searches the k-d tree from its root node recursively and checks whether the range $C_T$ on the (sub-)tree rooted at the current node $T$ has an intersection with the box $C_O$. If there is no intersection, the recursion returns directly without updating the tree (Line 4). If the box $C_T$ is fully contained in the box $C_O$, the box-wise delete set attributes $deleted$ and $treedeleted$ to true while the box-wise re-insertion set them to false by function \texttt{UpdateLazyLabel} (Line 6). The $pushdown$ attribute is set to true indicating that the latest incremental updates have not been applied to the offspring nodes of $T$. For the condition that $C_T$ intersects but not contained in $C_O$, the current point $P$ is firstly deleted from or re-inserted to the tree if it is contained in $C_O$ (Line 11), after which the algorithm looks into the son nodes recursively (Line 12-13) and updates all attributes of the current node $T$ (Line 15). Line 16-22 re-balance the tree if certain criterion is violated (Line 16) by re-building the tree in the same (Line 18) or a separate (Line 20) thread. The function \texttt{ViolateCriterion}, \texttt{Rebuild} and \texttt{ParrallelRebuild} are detailed in Section \ref{sec:balance}. 
        \begin{algorithm}[t]
        \SetAlgoLined
        \KwIn{$L$ \Comment Length of Downsample Cube\newline
        $P$ \Comment New Point}
            \SetKwFunction{boxsearch}{BoxwiseSearch}
            \SetKwFunction{boxdelete}{BoxwiseDelete}            
            \SetKwFunction{insertpoint}{PointwiseInsert}
            \SetKwFunction{nearest}{FindNearest}
            \SetKwFunction{getcenter}{Center}
            \SetKwFunction{cube}{FindCube}
            $C_D\leftarrow$ \cube{$L,P$} \\
            $P_{center}\leftarrow$ \getcenter{$C_D$};\\             
            $V\leftarrow$ \boxsearch{$RootNode, C_D$};\\
            $V.push(P)$;\\            
            $P_{nearest}\leftarrow$ \nearest$(V,P_{center})$;\\
            \boxdelete{$RootNode, C_D$}\\
            \insertpoint{$RootNode, P_{nearest}$};
        \caption{Downsample}
        \label{alg:downsample}
        \end{algorithm}        
        \begin{figure}[t]
            \setlength\abovecaptionskip{-0.1\baselineskip}
            \centering
            \includegraphics[width=0.48\textwidth]{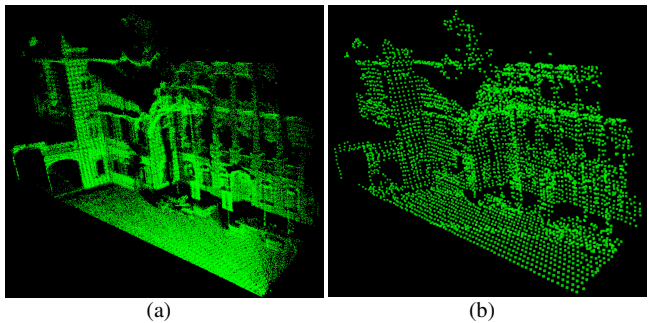}
            \caption{Point Cloud Downsample. (a): the point cloud before down-sampling. (b): the point cloud after down-sampling}
            \label{fig:downsample} 
            \vspace{-1.1cm}
        \end{figure}  
    	\subsubsection{Downsample}        
        Our ikd-Tree further supports down-sampling as detailed in {\bf Algorithm \ref{alg:downsample}}. For the given point $P$ and down-sampling resolution $L$, the algorithm partitions the space evenly into cubes of length $L$, then the box $C_D$ that contains point $P$ is found (Line 1). The algorithm only keeps the point that is nearest to the center $P_{center}$ of $C_D$ (Line 2). This is achieved by firstly searching all points contained in $C_D$ on the k-d tree and stores them in a point array $V$ together with the new point $P$ (Line 3-4). The nearest point $P_{nearest}$ is obtained by comparing the distances of each point in $V$ to the center $P_{center}$ (Line 5). Then existing points in $C_D$ are deleted (Line 6), after which the nearest point $P_{nearest}$ is inserted to the k-d tree (Line 7). The implementation of box-wise search is similar to the box-wise delete and re-insertion (see {\bf Algorithm \ref{alg:boxwise}}). An example of downsample is shown in Fig. \ref{fig:downsample}. 
        
        In summary, Table \ref{tab:incre} shows the comparison of supported incremental updates on the static k-d tree\cite{bentley1975kdtree}, the dynamic k-d tree\cite{bentley1980decomposable}, the scapegoat k-d tree\cite{scapegoat} and our ikd-Tree.

    	\setlength{\tabcolsep}{0.575em}
    	\begin{table}
        \renewcommand{\arraystretch}{1.2}    	
    	\centering
        \caption{Comparison of Supported Incremental Updates}
\begin{tabular}{|cc|c|c|c|c|}
\hline
\multirow{2}{*}{}                                                                            & \multirow{2}{*}{} & \multirow{2}{*}{\begin{tabular}[c]{@{}c@{}}Static\\ K-D Tree\end{tabular}} & \multirow{2}{*}{\begin{tabular}[c]{@{}c@{}}Dynamic\\ K-D Tree\end{tabular}} & \multirow{2}{*}{\begin{tabular}[c]{@{}c@{}}Scapegoat\\ K-D Tree\end{tabular}} & \multirow{2}{*}{ikd-Tree} \\
                                                                                             &                   &                                                                            &                                                                             &                                                                               &                           \\ \hline
\multicolumn{1}{|c|}{\multirow{3}{*}{\begin{tabular}[c]{@{}c@{}}Point-\\ wise\end{tabular}}} & Insert            & \xmark                                                                     & \cmark                                                                      & \cmark                                                                        & \cmark                    \\
\multicolumn{1}{|c|}{}                                                                       & Delete            & \xmark                                                                     & \xmark                                                                      & \cmark                                                                        & \cmark                    \\
\multicolumn{1}{|c|}{}                                                                       & Re-insert         & \xmark                                                                     & \xmark                                                                      & \xmark                                                                        & \cmark                    \\ \hline
\multicolumn{1}{|c|}{\multirow{3}{*}{\begin{tabular}[c]{@{}c@{}}Box-\\ wise\end{tabular}}}   & Insert            & \xmark                                                                     & \cmark                                                                      & \cmark                                                                        & \cmark                    \\
\multicolumn{1}{|c|}{}                                                                       & Delete            & \xmark                                                                     & \xmark                                                                      & \xmark                                                                        & \cmark                    \\
\multicolumn{1}{|c|}{}                                                                       & Re-insert         & \xmark                                                                     & \xmark                                                                      & \xmark                                                                        & \cmark                    \\ \hline
\multicolumn{2}{|c|}{Downsample}                                                                                 & \xmark                                                                     & \xmark                                                                      & \xmark                                                                        & \cmark                    \\ \hline
\end{tabular}
        \label{tab:incre}         
        \end{table}

\subsection{Re-balancing} \label{sec:balance}
    	Our ikd-Tree actively monitors the balance property of the incremental k-d tree and dynamically re-balance it by partial re-building.
    	
    	\subsubsection{Balancing Criterion}
    	The balancing criterion is composed of two sub-criterions: $\alpha$-balanced criterion and $\alpha$-deleted criterion. Suppose a sub-tree of the incremental k-d tree is rooted at $T$. The sub-tree is $\alpha$-balanced if and only if it satisfies the following condition:
        \addtolength{\abovedisplayskip}{-0.1cm}
        \addtolength{\belowdisplayskip}{-0.1cm}     	
    	\begin{equation}
    	\begin{aligned}
    	    S(T.leftson) &< \alpha_{bal} \Big(S(T)-1\Big)\\
    	    S(T.rightson) &< \alpha_{bal} \Big(S(T)-1\Big)
    	\end{aligned}
    	\label{eq:abalance}
    	\end{equation}
        \addtolength{\abovedisplayskip}{0.1cm}
        \addtolength{\belowdisplayskip}{0.1cm}     	
    	where $\alpha_{bal} \in (0.5,1)$ and $S(T)$ is the $treesize$ attribute of the node $T$. 
    	
    	The $\alpha$-deleted criterion of the sub-tree rooted at $T$ is
    	\begin{equation}
    	    \begin{aligned}
        	    I(T)< \alpha_{del} S(T)
    	    \end{aligned}
    	\label{eq:adeleted}
    	\end{equation}
    	where $\alpha_{del} \in(0,1)$ and $I(T)$ denotes the number of invalid nodes on the sub-tree (i.e., the attributes $invalidnum$ of node $T$). 
    	
    	If a sub-tree of the incremental k-d tree meets both criterion, the sub-tree is balanced. The entire tree is balanced if all sub-trees are balanced. Violation of either criterion will trigger a re-building process to re-balance that (sub-) tree: the $\alpha$-balanced criterion maintains the maximum height of the (sub-) tree. It can be easily proved that the maximum height of an $\alpha$-balanced tree is $\log_{1/\alpha_{bal}} (n)$ where $n$ is the tree size; the $\alpha$-deleted criterion ensures invalid nodes (i.e., labeled as ``deleted") on the (sub-) trees are removed to reduce tree size. Reducing height and size of the k-d tree allows highly efficient incremental operations and queries in future. The function \texttt{ViolateCriterion} in \textbf{Algorithm \ref{alg:boxwise}, Line 16} returns true if either criterion is violated. 
    	
    	\subsubsection{Re-build}\label{sec:singlerebuild}
        Assuming re-building is triggered on a subtree $\mathcal{T}$ (see Fig. \ref{fig:rebuild}), the sub-tree is firstly flattened into a point storage array $V$. The tree nodes labeled as ``deleted" are discarded during flattening. A new perfectly balanced k-d tree is then built with all points in $V$ by \textbf{Algorithm \ref{alg:build}}. 
        \begin{figure}
            \setlength\abovecaptionskip{-0.7\baselineskip}
            \centering
            \includegraphics[width = 0.5\textwidth]{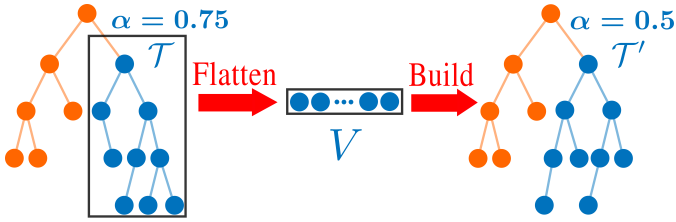}
            \caption{Re-build an unbalanced sub-tree}
            \vspace{-0.5cm}
            \label{fig:rebuild}            
        \end{figure}    	

    	\subsubsection{Parallel Re-build}
    	An evident degradation of real-time ability is observed when re-building a large sub-tree on the incremental k-d tree. To preserve high real-time ability, we design a double-thread re-building method: the main thread only re-builds sub-trees whose size is smaller than a predetermined value $N_{\max}$ and the second thread re-builds the rest. The key problem is how to avoid information lose and memory conflicts between the main thread and the second thread. 
    	
        The re-building algorithm on the second thread is shown in \textbf{Algorithm \ref{alg:rebuild}}. Denote the sub-tree to re-build in the second thread as $\mathcal{T}$ and its root node as $T$. The second thread will lock all incremental updates (i.e., points insert, re-insert, and delete) but not queries on this sub-tree (Line 2). Then the second thread copies all valid points contained in the sub-tree $\mathcal{T}$ into a point array $V$ (i.e. flatten) while leaving the original sub-tree unchanged for possible queries during the re-building process (Line 3). After the flattening, the sub-tree is unlocked for the main thread to take further requests of incremental updates (Line 4). These requests will be suspended and recorded in a queue named as operation logger. Once the second thread completes building a new balanced k-d tree $\mathcal{T}'$ from the point array $V$ (Line 5), the recorded update requests will be performed on the balanced sub-tree $\mathcal{T}'$ by function \texttt{IncrementalUpdates} (Line 6-8) where the parallel re-building option is set to false (as it is already in the second thread). After all pending requests are processed, the algorithm locks the node $T$ from both incremental updates and queries and replace it with the new one $T'$ (Line 9-12). Finally, the algorithm frees the memory of the original sub-tree (Line 13). Note that \texttt{LockUpdates} does not block queries, which can be conducted parallelly in the main thread. In contrast, \texttt{LockAll} blocks all access including queries, but it finishes very quickly (i.e., only one instruction), allowing timely queries in the main thread. The function \texttt{LockUpdates} and \texttt{LockAll} are implemented by mutual exclusion (mutex).
        \begin{algorithm}
        \SetAlgoLined
        \KwIn{$T$ \Comment Root node of $\mathcal{T}$ for re-building}
        \SetKwFunction{rebuild}{Rebuild}
        \SetKwFunction{build}{Build}
        \SetKwFunction{parrebuild}{ParallelRebuild}
        \SetKwFunction{flatten}{Flatten}
        \SetKwFunction{lockop}{LockUpdates}
        \SetKwFunction{lockall}{LockAll}
        \SetKwFunction{unlock}{Unlock}
        \SetKwFunction{release}{Free}
        \SetKwFunction{runop}{IncrementalUpdates}
        \SetKwFunction{sleep}{Sleep}
            \SetKwProg{Fn}{Function}{}{}
            \Fn{\parrebuild{$T$}}{
                \lockop{$T$};\\
                $V \leftarrow \flatten{T}$;\\
                \unlock{$T$};\\
                $T'\leftarrow \build{V,0,size(V)-1}$;\\
                \ForEach{ $op$ \textbf{in} $OperationLogger$}{
                    \runop{$T',op,false$}
                }
                $T_{temp} \leftarrow T$;\\                
                \lockall{$T$};\\                        
                $T \leftarrow T'$;\\
                \unlock{$T$};\\
                \release{$T_{temp}$}; \\
            }
            \textbf{End Function}
        \caption{Parallelly Rebuild for Re-balancing}
        \label{alg:rebuild}        
        \end{algorithm}


	\subsection{K-Nearest Neighbor Search}
	    The nearest search on the incremental k-d tree is an accurate nearest search \cite{friedman1975nearestsearch} instead of an approximate one as \cite{muja2009flannfast}. The function \texttt{Pushdown} is applied before searching the sub-tree rooted at node $T$ to pass down its lazy labels. We use the attribute $range$ to speed up the search process thus hard real-time ability is preserved. Due to the space limit, the details of k-nearest search algorithm is not presented in this paper. Interested readers can refer to the related codes in our open source library. 

    \section{Complexity Analysis}\label{sec:theory}
        \subsection{Time Complexity}
        The time complexity of ikd-Tree breaks into the time for incremental operations (insertion, re-insertion and delete) and re-building. 
        
        \subsubsection{Incremental Operations}
        The time complexity of point-wise operations is given as 
        \begin{lemma}[Point-wise Operations]
        \vspace{0.10cm}
        \textit{An incremental k-d tree can handle a point-wise incremental operation with time complexity of $O(\log n)$ where $n$ is the tree size.}
        \label{Lemma:pointwise}        
        \vspace{0.10cm}        
        \end{lemma}
        \begin{proof} The maximum height of an incremental k-d tree can be easily proved to be $\log_{1/\alpha_{bal}}(n)$ from Eq. (\ref{eq:abalance}) while that of a static k-d tree is $\log_2 n$. Hence the lemma is directly obtained from \cite{bentley1975kdtree} where the time complexity of point insertion and delete on a k-d tree was proved to be $O(\log n)$. The point-wise re-insertion modifies the attribute $deleted$ on a tree node thus the time complexity is the same as point-wise delete.
        \end{proof}
        
        The time complexity of box-wise operations on an incremental k-d tree is:
        \begin{lemma}[Box-wise Operations]
        \vspace{0.10cm}  
        \textit{An incremental 3-d tree handles box-wise insertion of $m$ points in $C_D$ with time complexity of $O(m\log n)$. Furthermore, suppose points on the 3-d tree are in space $S_x\times S_y\times S_z$ and $C_D = L_x\times L_y \times L_z$. The box-wise delete and re-insertion can be handled with time complexity of $O(H(n))$, where }
        \begin{equation}
            O(H(n)) = 
            \begin{cases}
                O(\log n) &  \text{if} \Delta_{\min} \geqslant \alpha(\frac{2}{3})\text{(*)}\\
                O(n^{1-a-b-c}) & \text{if} \Delta_{\max}\leqslant 1-\alpha(\frac{1}{3})\text{(**)}\\
                O(n^{\alpha(\frac{1}{3})-\Delta_{\min}-\Delta_{\text{med}}}) & \text{if (*) and (**) fail and}\\
                 & \Delta_{\text{med}}<\alpha(\frac{1}{3})-\alpha(\frac{2}{3})\\
                O(n^{\alpha(\frac{2}{3})-\Delta_{\min}}) & \text{otherwise.}
            \end{cases}
        \end{equation} 
        \textit{where $a = \log_n \frac{S_x}{L_x}$, $b = \log_n \frac{S_y}{L_y}$ and $c = \log_{n} \frac{S_z}{L_z}$ with $a,b,c\geqslant 0$. $\Delta_{\min}$, $\Delta_{\text{med}}$ and $\Delta_{\max}$ are the minimal, median and maximal value among $a$, $b$ and $c$. $\alpha (u)$ is the flajolet-puech function with $u\in [0,1]$, where particular value is provided: $\alpha(\frac{1}{3}) = 0.7162$ and $\alpha(\frac{2}{3})=0.3949$} 
        
        \begin{proof}
            The box-wise insertion is implemented by point-wise insertion thus the time complexity can be directly obtained from Lemma \ref{Lemma:pointwise}. An asymptotic time complexity for range search on a k-d tree is provided in \cite{chanzy2001rangesearch}. The box-wise delete and re-insertion can be regarded as a range search except that lazy labels are attached to the tree nodes. Therefore, the conclusion of range search can be applied to the box-wise delete and re-insertion on the incremental k-d tree.  
        \end{proof}
        \label{Lemma:boxwise}        
        \vspace{0.10cm}
        \end{lemma} 
        
        The down-sampling method on an incremental k-d tree is composed of box-wise search and delete followed by the point insertion. By applying Lemma \ref{Lemma:pointwise} and Lemma \ref{Lemma:boxwise}, the time complexity of downsample is $O(\log n)$+$O(H(n))$. Generally, the downsample hypercube $C_D$ is very small comparing with the entire space. Therefore, the normalized range $\Delta x$, $\Delta y$ and $\Delta z$ are small and the value of $\Delta_{\min}$ satisfies the condition (*) for time complexity of $O(\log n)$. Hence, the time complexity of down-sampling is $O(\log n)$.
        \subsubsection{Re-build} Time complexity for re-building breaks into two types: single-thread re-building and parallel two-thread re-building. In the former case, the re-building is performed by the main thread in a recursive way, each level takes the time of sorting (i.e., $O(n)$) and the total time over $\log n$ levels is $O(n\log n)$ \cite{bentley1975kdtree} when the the dimension $k$ is low (e.g., 3 in most robotic applications). For parallel re-building, the time consumed in the main thread is only flattening (which suspends the main thread from further incremental updates, \textbf{Algorithm \ref{alg:rebuild}}, Line 2-4) but not building (which is performed in parallel by the second thread, \textbf{Algorithm \ref{alg:rebuild}}, Line 5) or tree update (which takes constant time $O(1)$, \textbf{Algorithm \ref{alg:rebuild}}, Line 10-12), leading to a time complexity of $O(n)$. In summary, the time complexity of re-building an incremental k-d tree is $O(n)$ for two-thread parallel re-building and $O(n\log n)$ for single-thread re-building. 

        \subsubsection{Nearest Search}
        For robotic applications, the points dimension is usually very small. Hence the time complexity of k-nearest search on the incremental k-d tree can be simply approximated as $O(\log n)$ because the maximum height of the incremental k-d tree is maintained no larger than $\log_{\frac{1}{\alpha}}n$.        
        \subsection{Space Complexity}
        As shown Section \ref{sec:3}, each node on the incremental k-d tree records point information, tree size, invalid point number and point distribution of the tree. Extra flags such as lazy labels are maintained on each node for box-wise operations. For an incremental k-d tree with $n$ nodes, the space complexity is $O(n)$ though the space constant is a few times larger than a static k-d tree. 
    	\setlength{\tabcolsep}{0.9em}        
        \begin{table}
        \centering
        \caption{The Parameters Setup of ikd-Tree in Experiments}
        \begin{tabular}{|c|c|c|}
        \hline
                       & Randomized Data & LiDAR Inertial-Odometry Mapping \\ \hline
        $N_{\max}$     & 1500            & 1500                            \\
        $\alpha_{bal}$ & 0.6             & 0.6                             \\
        $\alpha_{del}$ & 0.5             & 0.5                             \\
        $C_D$          & -               & $0.2m\times 0.2m\times 0.2$     \\ \hline
        \end{tabular}
        \label{tab:param}  
        \vspace{-0.2cm}
        \end{table} 
    \section{Application Experiments}\label{sec:experiments}

        \subsection{Randomized Data Experiments}
        The efficiency of our ikd-Tree is fully investigated by two experiments on randomized incremental data sets. The first experiment generates 5,000 points randomly in a $10m\times10m\times10m$ space (i.e., the workspace) to initialize the incremental k-d tree. Then 1,000 test operations are conducted on the k-d tree. In each test operation, 200 new points randomly sampled in the workspace are inserted (point-wise) to the kd-tree. Then another 200 points are randomly sampled in the workspace and searched on (but not inserted to) the k-d tree for 5 nearest points of each. For every 50 test operations, 4 cubes are sampled in the workspace with side length of $1.5m$ and points contained in these 4 cubes are deleted (box-wise) from the k-d tree.  For every 100 test operations, 2,000 new points are sampled in the workspace and inserted (point-wise) to the k-d tree.  We compare the ikd-Tree with the static k-d tree used in point cloud library\cite{rusu20113d} where at each test operation the k-d tree is entirely re-built. The experiments are performed on a PC with Intel i7-10700 CPU at 2.90GHz and only 2 threads running. The parameters of the incremental k-d tree are summarized in Table \ref{tab:param} where no down-sampling is used to allow a fair comparison. Also the maximal point number allowed to store on the leaf node of a static k-d tree is set to 1 while the original setting in point cloud library is 15. 
          
        \begin{figure}[t]
            \setlength\abovecaptionskip{-0.1\baselineskip}
            \centering
            \includegraphics[width=0.485\textwidth]{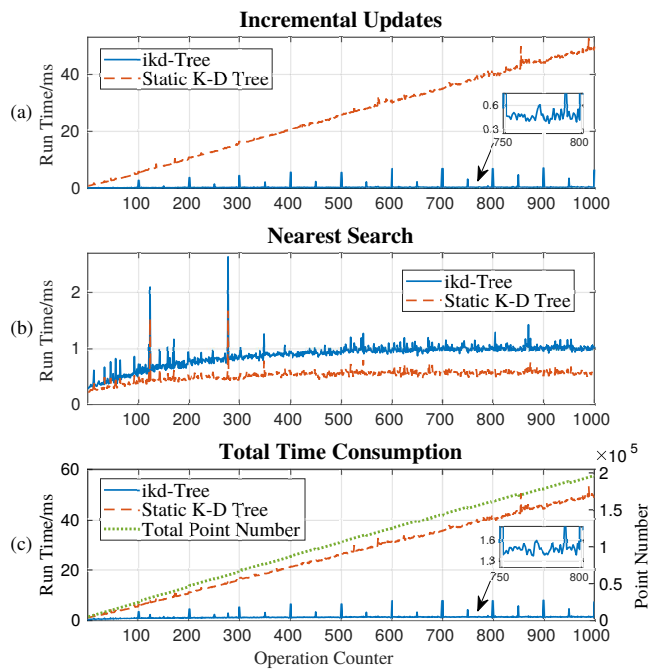}
            \caption{The time performance comparison between an ikd-Tree and a static k-d tree.}
            \label{fig:randomized_exp_combine}    
            \vspace{-0.4cm}
        \end{figure} 
        
        The results of the first experiment are shown in Fig. \ref{fig:randomized_exp_combine}, where the point number increases from 5,000 to approximate 200,000. In this process, the time for incremental updates (including both incremental operations and re-building) on the ikd-Tree remains stably around 1.6 $ms$ while that for the static k-d tree grows linearly with number of the points (see Fig. \ref{fig:randomized_exp_combine}(a)). The high peaks in the time consumption are resulted from the large-scale point-wise insertion (and associated re-balancing) and the low peaks are resulted from box-wise delete (and associated re-balancing). As shown in Fig. \ref{fig:randomized_exp_combine}(b), the time performance of the k-nearest search on the ikd-Tree is slightly slower than an static k-d tree, possibly due to the highly optimized implementation of the PCL library. Despite of the slightly lower efficiency in query, the overall time consumption of ikd-Tree outperforms the static k-d tree by one order of magnitude (See  Fig. \ref{fig:randomized_exp_combine}(c)). 

        \begin{figure}[t]
            \setlength\abovecaptionskip{-0.1\baselineskip}
            \centering
            \includegraphics[width=0.485\textwidth]{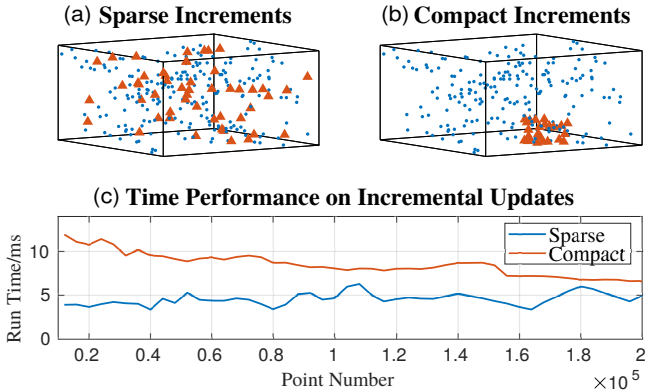}
            \caption{ Fig. (a) and (b) illustrate new points (orange triangles) and points already on the k-d tree (blue dots). Fig. (c) shows the time for incremental updates of sparse and compact data on k-d trees of different size.}
            \vspace{-0.5cm}
            \label{fig:randomized_exp_incre} 
        \end{figure}  

        The second experiment investigates the time performance of incremental updates for new points of different distribution. In the experiment, we sample two sets of 4,000 new points in a $10m\times10m\times10m$ space (i.e., the workspace): one is evenly distributed (i.e., sparse data, see Fig. \ref{fig:randomized_exp_incre} (a)) and the other concentrated in a $2.5m\times2.5m\times2.5m$ space (i.e., compact data, see Fig. \ref{fig:randomized_exp_incre} (b)). The sparse and compact data are inserted to an existing incremental k-d tree of different size but all sampled in the workspace. Fig. \ref{fig:randomized_exp_incre}(c) shows the running time of sparse and compact point-wise insertion on k-d trees of different size. As expected, the incremental updates for compact data are slower than the sparse one because re-building are more likely to be triggered when inserting a large amount of points into a small sub-tree of a k-d tree.

        \subsection{LiDAR Inertial-Odometry and Mapping}
        
        We test our developed ikd-Tree in an actual robotic application: lidar-inertial odometry (and mapping) presented in \cite{nuchter2007cachedKDtreeICP,segal2009generalizedicp, zhang2014loam, shan2018lego, lin2020loamlivox, xu2020fastlio}. In this application, k-d tree-based nearest points search is crucial to match a point in a new LiDAR scan to its correspondences in the map (or the previous scan). Since the map is dynamically growing by matching and merging new scans, the k-d tree has to be re-built every time a new scan is merged. Existing methods \cite{nuchter2007cachedKDtreeICP,segal2009generalizedicp, zhang2014loam, shan2018lego, lin2020loamlivox, xu2020fastlio} commonly used static kd-tree from PCL and rebuilds the entire tree based on all points in the map (or a submap). This leads to a significant computation costs severely limiting the map update rate (e.g., from $1Hz$ \cite{zhang2014loam, shan2018lego, lin2020loamlivox} to $10Hz$ \cite{xu2020fastlio}). 
        
        In this experiment, we replace the static k-d tree (build, update, and query) by our ikd-Tree, which enables incremental update of the map by updating the new points only.  We test ikd-Tree on the lidar-inertial mapping package FAST\_LIO in \cite{ xu2020fastlio}\footnote{\url{https://github.com/hku-mars/FAST\_LIO}}. The experiment is conducted on a real-world outdoor scene using a Livox Avia LiDAR\footnote{\url{https://www.livoxtech.com/de/avia}} \cite{liu2020low} with 70\textdegree FoV and a high frame rate of 100 $Hz$. All the algorithm is running on the DJI Manifold 2-C\footnote{ \url{https://www.dji.com/manifold-2/specs}} with a 1.8 $GHz$ quad-core Intel i7-8550U CPU and 8 $GB$ RAM.
        \begin{figure}[t]
            \setlength\abovecaptionskip{-0.7\baselineskip}        
            \centering
            \includegraphics[width=0.495\textwidth]{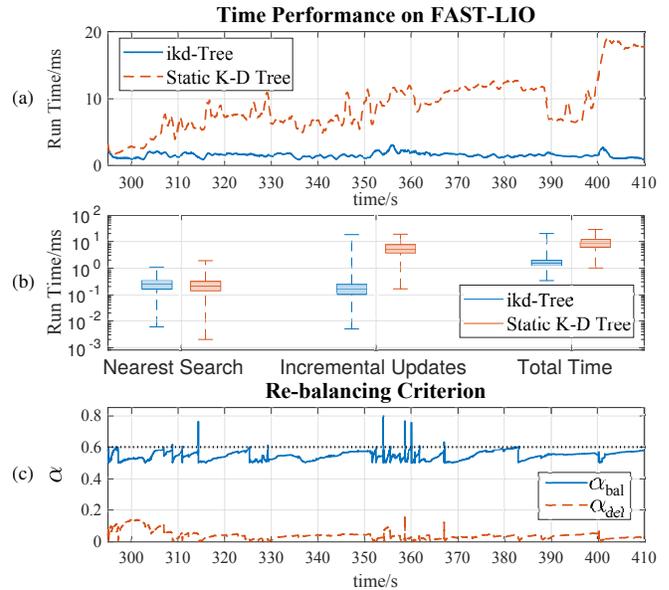}
            \caption{Fig. (a) shows the average running time of fusing one new lidar scan in FAST\_LIO using the ikd-Tree and a static k-d tree. Fig. (b) shows time for nearest search, incremental updates, and total time in fusing one lidar scan. Fig. (c) shows the balance property after re-building on main thread.}
            \label{fig:fastlio_exp_boxplot}            
        \end{figure}        
        \begin{figure}[t]
            \setlength\abovecaptionskip{-0.1\baselineskip}        
            \centering
            \includegraphics[width=0.485\textwidth]{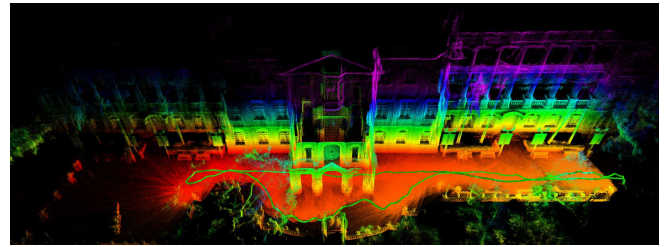}
            \caption{Mapping Result of the Main Building, University of Hong Kong. The green line is the path of the lidar computed by FAST\_LIO.}
            \label{fig:mappingresult}              
        \end{figure}  
         
         The time of fusing a new lidar scan in FAST\_LIO is shown in Fig. \ref{fig:fastlio_exp_boxplot}(a). The time is the averaged time of the most recent 100 scans. It is seen that the ikd-Tree achieves a nearly constant time performance around 1.6 $ms$, which fundamentally enables a mapping rate up to $100Hz$ (against $10Hz$ presented in the original work \cite{xu2020fastlio} using static k-d tree). On the other hand, the time with the static k-d tree is overall increasing linearly, and decreases occasionally due to the small overlap between the lidar current FoV and the map. The resultant processing time with the static k-d tree exceeds 10 $ms$ from 366$s$ on, which is more than the collection time of one lidar scan. 

         The time for fusing a new lidar scan consists of many operations, such as point registration, state estimation, and kd-tree related operations (including queries and update). The time breakdown of kd-tree related operations is shown in Fig. \ref{fig:fastlio_exp_boxplot}(b). The average time of incremental updates for ikd-Tree is 0.23 $ms$ which is only 4\% of that using an static k-d tree (5.71 $ms$). The average time of nearest search using the ikd-Tree and an static k-d tree are at the same level.
         
         Furthermore, Fig. \ref{fig:fastlio_exp_boxplot}(c) investigates the balancing property of the incremental kd-tree by examining the two criterions $\alpha_{bal}$ and $\alpha_{del}$, which are defined as:
        \begin{equation}
            \begin{aligned}
                \alpha_{bal}(T) &= \frac{\max\Big\{S(T.leftson),S(T.rightson)\Big\}}{S(T)-1}\\
                \alpha_{del}(T) &= \frac{I(T)}{S(T)}                
            \end{aligned}
        \end{equation}
        As expected, the two criterion are maintained below the prescribed thresholds (see Table \ref{tab:param}) due to re-building, indicating that the kd-tree is well-balanced through the incremental updates. The peaks over the thresholds is resulted from parallely re-building, which drop quickly when the re-building is done. Finally, Fig. \ref{fig:mappingresult} shows the $100Hz$ mapping results.

    \section{Conclusion}\label{sec:conclusion}
    This paper proposed an efficient data structure, ikd-Tree, to incrementally update a k-d tree in robotic applications. The ikd-Tree supports incremental operations in robotics while maintaining balanced by partial re-building. We provided a complete analysis of time and space complexity to prove the high efficiency of the proposed dynamic structure. The ikd-Tree was tested on a randomized experiment and an outdoor LiDAR odometry and mapping experiment. In all tests, the proposed data structure achieves two orders of magnitude higher efficiency.

	\bibliography{ral-kdtree}
	
\end{document}